\documentclass[journal, 11pt, onecolumn]{IEEEtran}
\usepackage{amsfonts}%
\usepackage{amssymb}%
\usepackage{setspace}
\usepackage[dvips]{graphicx}%
\usepackage{amsmath}%
\usepackage[usenames]{color}
\usepackage{algorithmic}
\usepackage{cite}
\usepackage{array}
\newtheorem{theorem}{Theorem}
\newtheorem{corollary}{Corollary}
\newtheorem{lemma}{Lemma}
\newtheorem{assumption}{Assumption}

\newtheorem{remark}{Remark}
\allowdisplaybreaks 

\DeclareMathOperator*{\argmin}{arg\,min}
\DeclareMathOperator*{\argmax}{arg\,max}

\newcommand{\comment}[1]{}

\ifodd 0
\newcommand{\remove}[1]{}
\newcommand{\add}[1]{#1}
\else
\newcommand{\remove}[1]{#1}
\newcommand{\add}[1]{}
\fi

\ifodd 1
\newcommand{\rmv}[1]{}
\else
\newcommand{\rmv}[1]{{\color{red}#1}}
\fi

\ifodd 1 
\newcommand{\armv}[1]{}
\else
\newcommand{\armv}[1]{{\color{red}#1}}
\fi
\ifodd 0
\newcommand{\rev}[1]{{\color{blue}#1}} 
\newcommand{\cem}[1]{{\color{magenta}#1}} 
\newcommand{\com}[1]{\textbf{\color{red}(COMMENT: #1)}} 
\newcommand{\clar}[1]{\textbf{\color{green}(NEED CLARIFICATION: #1)}}

\else
\newcommand{\rev}[1]{#1}
\newcommand{\cem}[1]{#1}

\newcommand{\com}[1]{}
\newcommand{\clar}[1]{}
\fi

\begin{document}
\title{Distributed Online Big Data Classification Using Context Information}
\author{\IEEEauthorblockN{Cem Tekin, Mihaela van der Schaar\\}
\IEEEauthorblockA{Department of Electrical Engineering,
University of California, Los Angeles\\
Email: \{cmtkn, mihaela\}@ucla.edu}
}

\maketitle

\begin{abstract}
Distributed, online data mining systems have emerged as a result of applications requiring analysis of large amounts of correlated and high-dimensional data produced by multiple distributed data sources. 
\rmv{To correctly identify the events and phenomena of interest, different classification functions are required in order to process data exhibiting different \cem{features/}characteristics.}
\rmv{However, a learner may
not have access to the complete set of necessary classification functions to process a certain type of data. 
Moreover, the accuracy of each classification function for the incoming data stream is unknown a priori and
needs to be learned online.}
\cem{We propose a} distributed online data classification framework
where data is gathered by distributed data sources and processed by a heterogeneous set of distributed learners which
learn online, at run-time, 
how to classify the different data streams either by using their locally available classification functions or by helping each other by classifying each other's data. 
\rev{Importantly, since the data is gathered at different locations, sending the data to another learner to process incurs additional costs such as delays, and hence this will be only beneficial if the benefits obtained from a better classification will exceed the costs.}
\armv{
We assume that the classification functions available to each processing element are fixed, but their prediction
accuracy for various types of incoming data are unknown and can change dynamically over time, and thus they need to be learned online.
}
%
We model \cem{the problem of joint classification by the
distributed and heterogeneous learners from multiple data sources}
as a distributed contextual
bandit problem where each data 
is characterized by
a specific context. We develop a distributed online learning algorithm for which we can prove sublinear regret\rmv{, i.e. the
average error probability converges to the error probability of the best distributed classification scheme given the context information}.
\rmv{Our bounds hold uniformly over time, without requiring any assumptions about of the types of classification functions used.}
Compared to prior work in distributed online data mining, our work is the first to provide analytic regret results characterizing the performance of the proposed algorithm.
\comment{\cem{We also relate our distributed contextual learning approach to the notion of concept drift, which was introduced to address non-stationary learning problems over time, and we show that sublinear regret can be achieved when the concept drift is gradual without requiring a drift detection mechanism.}}
\rmv{Finally, we illustrate our proposed solutions using \cem{distributed} online data mining systems for network security and compare our results with existing state of the art solutions for online data mining.}
\end{abstract}

\add{\vspace{-0.2in}}
\section{Introduction}\label{sec:intro}

A plethora of Big Data applications (network security, surveillance, health monitoring etc.) are emerging which
require online classification of large data sets collected from distributed network and traffic monitors, multimedia sources, sensor networks, etc.
This data is heterogeneous and dynamically evolves over time.
%
In this paper, we introduce a distributed online learning framework for classification of high-dimensional data
collected by distributed data sources. 

\cem{The distributedly collected data is processed by a set of decentralized heterogeneous learners equipped with classification functions with unknown accuracies.} \rev{In this setting communication, computation and sharing costs make it infeasible to use centralized data mining techniques where a single learner can access the entire data set.}
\rev{For example, in a wireless sensor surveillance network, nodes in different locations collect different information about different events. 
The learners/machines at each node of the network may run different classification algorithms, may have different resolution, processing speed, etc. 
\rmv{An event may happen rarely in one location while frequently in another location. Therefore, if the context implies that the event is a rare event, then it can be sent to a learner for which this event happens frequently to be classified.}
%
}

%
%
The input data stream and its associated context can be time-varying and heterogeneous. \cem{We use the term ``context'' generically, to represent}
any information related to the input data stream such as time, location and type \cem{(e.g., data features/characteristics/modality)} information.
Each learner can process (label) the incoming data in two different ways: either it can 
exploit its own information and its own classification functions or it can forward its input stream to another learner (possibly by incurring some cost) to have it labeled. A learner learns the accuracies of its own classification functions or other learners in an online way by comparing the result of the predictions with the true label of its input stream which is revealed at the end of each slot. The goal of each learner is to maximize its long term expected total reward, which is the expected number of correct labels minus the costs of classification. In this paper the cost is a generic term that can represent any known cost such as processing cost, delay cost, communication cost, etc. Similarly, data is used as a generic term. It can represent files of several Megabytes size, chunks of streaming media packets or contents of web pages.
\cem{A key differentiating
feature of our proposed approach is the focus on how the context information of the captured data can be utilized to maximize the
classification performance of a distributed data mining system.}
\rev{We consider cooperative learners which classify other's data when requested, but instead of maximizing the system utility function, a learner's goal is to maximize its individual utility. However, it can be shown that when the classification costs capture the cost to the learner which is cooperating with another learner to classify its data, maximizing the individual utility corresponds to maximizing the system utility.}

\cem{To jointly optimize the performance of the distributed data mining system, we} design a distributed online learning algorithm whose long-term average reward converges to the best distributed solution which can be obtained for the classification problem given complete knowledge of online data characteristics as well as their
classification function accuracies and costs when applied to this data.
We define the regret as the difference between the expected total reward of the best distributed classification scheme given complete knowledge about classification function accuracies and the expected total reward of the algorithm used by each learner. 
We prove a sublinear upper bound on the regret,
which implies that the average reward converges to the optimal average reward. The upper bound on regret gives a lower bound on convergence rate to the optimal average reward.
\armv{Specifically, we show that when the contexts of different streams arriving to the same learner are uniformly distributed over the context space,
the regret 
depends on the dimension of the context space, while when the context from different streams arriving to the same learner is concentrated in a small region of the context space, 
the regret is independent of the dimension of the context space.
} 
\rmv{We also give bounds on the computational complexity and memory requirements which are necessary to implement these algorithms, and show that they scale sublinearly with time. On the other hand, the optimal solution which could by derived by an ``oracle'' having
complete knowledge about classification functions of other learners and their accuracies may have required unbounded memory.}
\rmv{Summarizing, the proposed classification algorithms can be implemented in a distributed way, and have lower computational complexity than the exact solution.} 
\comment{Besides the theoretical results, we show that our distributed contextual learning framework can be used to deal with {\em concept drift} \cite{minku2010impact}, which occurs when the distribution of problem instances changes over time. 
Big data applications are often characterized by concept drift,
in which trending topics change rapidly over time.}
To illustrate our approach, we provide numerical results by applying our learning algorithm to the classification of network security data and compare the results with
existing state-of-the-art solutions.
\rmv{For example, a network security application needs to analyze several Gigabytes of data generated by different locations and/or at different time in order to detect malicious network behavior (see e.g., \cite{ishibashi2005detecting}). The context in this case can be the time of the day (since the network traffic depends on the time of the day) or it can be the IP address of the machine that sent the data (some locations may be associated with higher malicious activity rate) or context can be two dimensional capturing both the time and the location. 
%
%
In our model since the classification accuracies are not known a priori, the network security application needs to learn which one to select based on the context information available about the network data. We note that our online learning framework does not require any prior knowledge about the network traffic characteristics or network topology but the security application learns the best actions from its past observations and decisions. In another example, context can be the information about a priori probability about the origin of the data that is send to the network manager by routers in different locations.
}

The remainder of the paper is organized as follows. In Section \ref{sec:related} we describe the related work and highlight the differences from our work. In Section \ref{sec:probform} we describe the decentralized data classification problem, the optimal distributed classification scheme given the complete system model, its computational complexity, and the regret of a learning algorithm with respect to the optimal classification scheme. 
Then, we consider the model with unknown system statistics and propose a distributed online learning algorithm with {\em uniform contextual partitioning} in Section \ref{sec:iid}. 
\armv{In Section \ref{sec:zooming}, we develop another learning algorithm with {\em adaptive distributed contextual zooming} whose regret can be much better than the previous algorithm depending on the context arrival process.
Several extensions to our proposed learning algorithms are given 
%
%
in Section \ref{sec:discuss}.}
Using a network security application we provide numerical results on the performance of our distributed online learning algorithm in Section \ref{sec:numerical}. Finally, the concluding remarks are given in Section \ref{sec:conc}.
\add{\vspace{-0.2in}}
\section{Related Work} \label{sec:related}

Related work can be divided into two categories: Online learning
for data mining and multi-armed bandit methods aimed at
learning how to act.

Online learning in distributed data classification systems aims to address the informational decentralization, communication costs and privacy issues arising in these systems. 
\armv{Specifically, in online ensemble learning techniques, the predictions of decentralized and heterogeneous classifiers are combined to improve the classification accuracy.
In these systems, each classifier learns at different rates because either each learner observes the entire feature space but has access to a subset of instances of the entire data set, which is called {\em horizontally distributed} data, or each learner has access to only a subset of the features but the instances can come from the entire data set, which is called {\em vertically distributed} data.  
}
For example in \cite{predd2006distributed, perez2010robust, breiman1996bagging, wolpert1992stacked}, various solutions are proposed for distributed data mining problems of horizontally distributed data, while in \cite{zheng2011attribute, yubig2013} ensemble learning techniques are developed that exploit the correlation between the local learners for vertically distributed data. Several cooperative distributed data mining techniques are proposed in \cite{mateos2010distributed, chen2004channel, kargupta1999collective, yubig2013}, where the goal is to improve the prediction accuracy with costly communication between local predictors. In this paper, we take a different approach: instead of focusing on the characteristics of a specific data stream, we focus on the characteristics of data streams with the same context information.
\cem{This novel approach allows us to deal with both horizontally and vertically distributed data in a unified manner within a distributed data mining system.}
%
%
Although our framework and illustrative results are depicted using horizontally distributed data, if context is changed to be the set of relevant features, then our framework and results can operate on vertically distributed data.
Moreover, we assume no prior knowledge of the data and context arrival processes and classification function accuracies, and the learning is done in a non-Bayesian way.
\armv{
Learning in a non-Bayesian way is appropriate in decentralized system since learners often do not have correct beliefs about the distributed system dynamics.}
%
%

\rev{Most of the prior work in distributed data mining provides algorithms which are asymptotically converging to an optimal or locally-optimal solution without providing any rates of convergence.}
On the contrary, we do not only prove convergence results, but we are also able to explicitly characterize the performance loss incurred at each time step with respect to the optimal solution. In other words, we prove regret bounds that hold uniformly over time. Some of the existing solutions (including \cite{sewell2008ensemble, alpaydin2004introduction, mcconnell2004building, breiman1996bagging, wolpert1992stacked, buhlmann2003boosting, lazarevic2001distributed, perlich2011cross}) propose ensemble learning techniques including bagging, boosting, stacked generalization and cascading, where the goal is to use classification results from several classifiers to increase the prediction accuracy. 
In our work we only consider choosing the best classification function (initially unknown) from a set of classification functions that are accessible by decentralized learners. However, our proposed distributed learning method can easily be adapted to perform ensemble learning.
We provide a detailed comparison to our work in Table \ref{tab:comparison1}.

\comment{Our contextual framework can also deal with concept drift \cite{minku2010impact}. Formally, a concept is the distribution of the problem, \cem{i.e., the joint distribution of the input data stream, true labels and context information,}
at a certain point of time \cite{narasimhamurthy2007framework}. Concept drift is a change in this distribution \cite{gama2004learning, gao2007appropriate}. By treating time as the context, the same regret bounds of our learning algorithms hold under concept drift, without requiring any drift detection mechanism as required in some of the existing solutions employed for dealing with concept drift \cite{baena2006early, minku2012ddd}. Hence, our proposed framework can be used to formalize and solve the problem of concept drift
which was mainly dealt with previously in an ad-hoc manner.}

Other than distributed data mining, our learning framework can be applied to any problem that can be formulated as a decentralized contextual bandit problem. Contextual bandits have been studied before in \cite{slivkins2009contextual, dudik2011efficient, langford2007epoch, chu2011contextual} in a single agent setting, where the agent sequentially chooses from a set of alternatives with unknown rewards, and the rewards depend on the context information provided to the agent at each time step. To the best of our knowledge, our work is the first to address the decentralized contextual bandit problem in a system of cooperative learning agents. 
In \cite{li2010contextual}, a contextual bandit algorithm named LinUCB is proposed for recommending personalized news articles,
which is variant of the UCB algorithm \cite{auer} designed for linear payoffs.
Numerical results on real-world Internet data are provided, but no
theoretical results on the resulting regret are derived.
A perceptron based algorithm is used with upper confidence bounds in \cite{crammer2011multiclass} in a centralized single user setting that achieves sublinear regret when the instances are chosen by an adversary and the learning algorithm receives binary feedback about the true label instead of the true label itself. Previously, distributed multi-user learning is only considered in the standard finite armed bandit problem. In \cite{anandkumar, hliu1} distributed online learning algorithms that converge to the optimal allocation with logarithmic regret are proposed, given that the optimal allocation is an orthogonal allocation in which each user selects a different action. This is generalized in \cite{tekin2012sequencing} to dynamic resource sharing problems and logarithmic regret results are also proved for this case. 
Alternatively, in this paper, we consider distributed online learning in a contextual bandit setting. We provide a detailed comparison between our work and related work in multi-armed bandit learning in Table \ref{tab:comparison2}. Our decentralized contextual learning framework can be seen as an important extension of the centralized contextual bandits framework \cite{slivkins2009contextual}. The main difference is that: (i) a three phase learning algorithm with {\em training}, {\em exploration} and {\em exploitation} phases are needed instead of the standard two phase, i.e., {\em exploration} and {\em exploitation} phases, algorithms used in centralized contextual bandit problems; (ii) the adaptive partitions of the context space should be formed in a way that each learner can efficiently utilize what is learned by other learners about the same context. In the distributed contextual framework, the training phase is necessary since the context arrivals to learners are different which makes the learning rates of the learners for different context different.
\armv{Essentially, the training phase balances the learning rates of the learners.}

\begin{table}[t]
\centering
{\renewcommand{\arraystretch}{0.6}
{\fontsize{8}{6}\selectfont
\setlength{\tabcolsep}{.1em}
\begin{tabular}{|l|c|c|c|c|c|}
\hline
&  \cite{breiman1996bagging, buhlmann2003boosting, lazarevic2001distributed, chen2004channel, perlich2011cross} & \cite{mateos2010distributed, kargupta1999collective} &  \cite{zheng2011attribute} & This work \\
\hline
Aggregation & non-cooperative & cooperative & cooperative & \rev{no} \\
\hline
Message  & none & data & training  & data and label \\
exchange & & & residual & only if improves  \\
& & & &   performance \\
\hline
Learning  & offline/online & offline & offline & Non-bayesian \\
approach&&&& online\\
\hline
Correlation & N/A & no & no & yes\\
exploitation & & & &\\
\hline
Information from  & no & all & all & only if improves  \\
other learners & & & &  accuracy \\
\hline
Data partition & horizontal & horizontal & vertical & horizontal \\
\hline
Bound on regret,  & no &no &no &yes - sublinear\\
convergence rate &&&&\\
\hline
\end{tabular}
}
}
\caption{Comparison with related work in distributed data mining}
\label{tab:comparison1}
\add{\vspace{-0.1in}}
\end{table}

\begin{table}[t]
\centering
{\fontsize{8}{6}\selectfont
\setlength{\tabcolsep}{.25em}
\vspace{-0.2in}
\begin{tabular}{|l|c|c|c|c|c|}
\hline
&\cite{slivkins2009contextual, dudik2011efficient, langford2007epoch, chu2011contextual} &  \cite{hliu1, anandkumar, tekin2012sequencing} & \cite{tekin4} & This work \\
\hline
Multi-user & no & yes & yes & yes \\
\hline
Cooperative & N/A & yes & no & yes \\
\hline
Contextual & yes & no & no & yes \\
\hline
Data arrival  & arbitrary & i.i.d. or Markovian & i.i.d. & i.i.d or arbitrary \\
process& & & & \\
\hline
Regret & sublinear & logarithmic & may be linear & sublinear \\
\hline
\end{tabular}
}
\caption{Comparison with related work in multi-armed bandits}
\vspace{-0.35in}
\label{tab:comparison2}
\end{table}

\add{\vspace{-0.2in}}
\section{Problem Formulation}\label{sec:probform}

The system model is shown in Figure \ref{fig:system}. There are $M$ learners which are indexed by the set ${\cal M} = \{1,2,\ldots,M\}$.
Let ${\cal M}_{-i} = {\cal M} - \{i\}$.
\cem{These learners work in a discrete time setting $t=1,2,\ldots,T$, where the following events happen sequentially, in each time slot: (i) a data stream $s_i(t)$ with a specific context $x_i(t)$ arrives to each learner $i \in {\cal M}$, (ii) each learner chooses one of its own classification functions or another learner to send its data and context, and produces a label based on the prediction of its own classification function or the learner to which its sent its data and context, (iii) the truth (true label) is revealed eventually, perhaps by events or by a supervisor, only to the classifier where the data arrived.
}
%
%

%
Each learner $i \in {\cal M}$ has access to a set of classification functions ${\cal F}_i$ which it can invoke to classify the data. 
Classifier $i$ knows the functions in ${\cal F}_i$ and costs of calling them\footnote{Alternatively, we can assume that the costs are random variables with bounded support whose distribution is unknown. In this case, the learners will not learn the accuracy but they will learn accuracy minus cost.}, but not their accuracies, while it knows the set of other learners ${\cal M}_{-i}$ and costs of calling them but does not know the functions ${\cal F}_k$, $k \in {\cal M}_{-i}$, but only knows an upper bound on the number of classification functions that each learner has, i.e., $F_{\max}$ on $|{\cal F}_k|$\footnote{For a set $A$, let $|A|$ denote the cardinality of that set.}, $k \in {\cal M}_{-i}$.
%
Classifier $i$ can either invoke one of its classification functions or forward the data to another learner to have it labeled. We assume that for learner $i$ calling each classification function $k \in {\cal F}_i$ incurs a cost $d_k$.
For example, if the application is delay critical this can be the delay cost, or this can represent the computational cost and power consumption associated with calling a classification function.
Since the costs are bounded, without loss of generality we assume that costs are normalized, i.e., $d_k \in [0,1]$.
We assume that a learner can only call a single function for each input data in order to label it. This is a reasonable assumption when the application is delay sensitive since calling more than one classification function increases the delay. 
\rmv{However, our framework can also be easily extended to ensemble learning where results of different classification function are combined to increase the prediction accuracy. We discuss this possible extension in Section \ref{sec:conc}.}
A learner $i$ can also send its input to another learner in ${\cal M}_{-i}$ in order to have it labeled. Because of the communication cost and the delay caused by processing at the recipient, we assume that whenever a data stream is sent to another learner $k \in {\cal M}_{-i}$ a cost of $d_k$ is incurred. \rev{The learners are cooperative which implies that learner $k \in {\cal M}_{-i}$ will return a label to $i$ when called by $i$. Similarly, when called by $k \in {\cal M}_{-i}$, learner $i$ will return a label to $k$. We do not consider the effect of this on $i$'s learning rate, however, since our results hold for the case when other learners are not helping $i$ to learn about its own classification functions, they will hold when other learners help $i$ to learn about its own classification functions. 
If we assume that $d_k$ also captures the cost to learner $k$ to classify and sent the label to learner $i$, then maximizing $i$'s own utility corresponds to maximizing the system utility.} 
\rev{
%
Let ${\cal K}_i = {\cal F}_i \cup {\cal M}_{-i}$. We call ${\cal K}_i$ the set of arms (alternatives).
}

\begin{figure}
\begin{center}
\includegraphics[width=0.9\columnwidth]{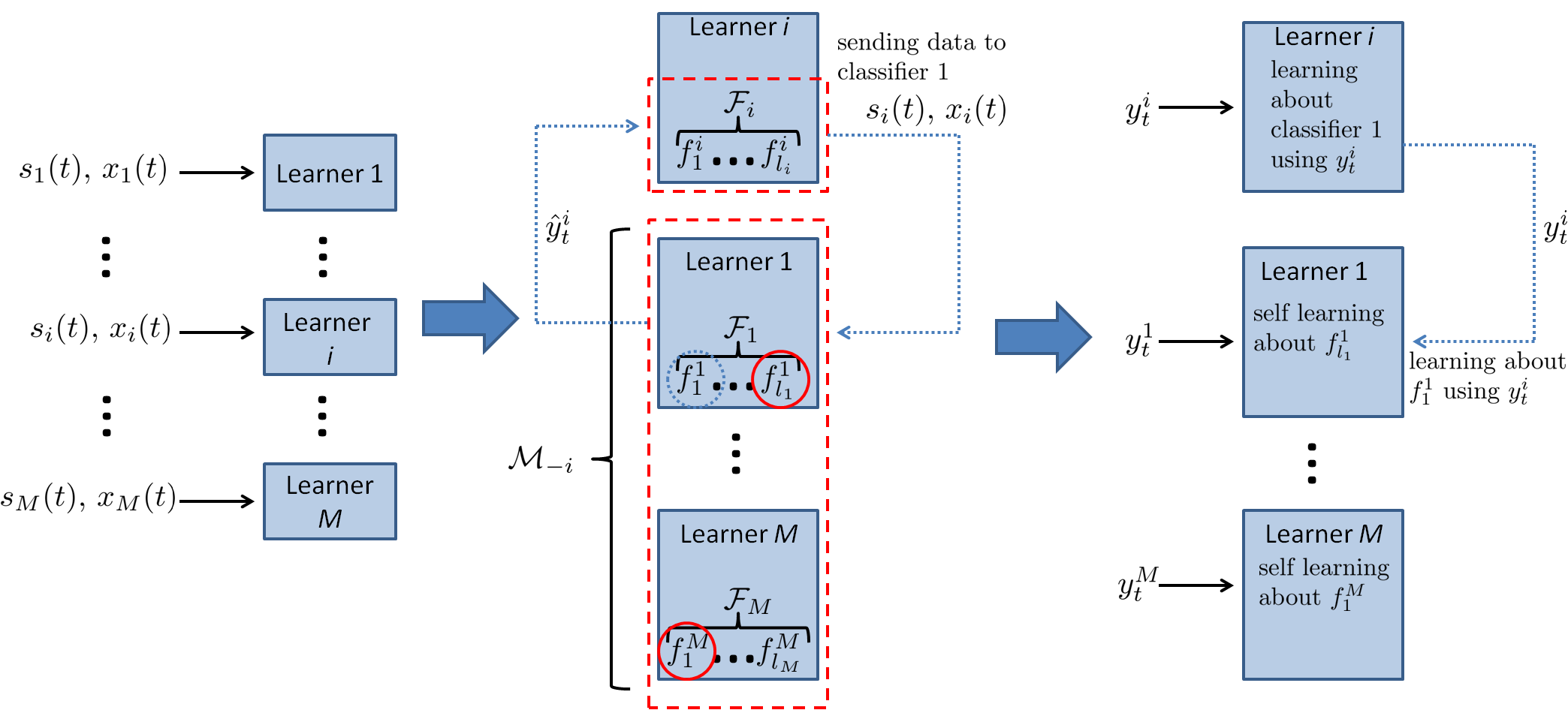}
\vspace{-0.1in}
\caption{Operation of the distributed data classification system from the viewpoint of learner 1} 
\vspace{-0.4in}
\label{fig:system}
\end{center}
\end{figure}

We assume that each \rev{classification function} \com{Give a concrete example of this function} produces a binary label\footnote{In general we can assume that labels belong to $\mathbb{R}$ and define the classification error as the mean squared error or some other metric. Our results can be adapted to this case as well.}. 
%
%
\armv{Considering only binary classifiers is not restrictive since in general, ensembles of binary classifiers can be used to accomplish more complex classification tasks \cite{lienhart2003detector, mao2005multiclass}.}
The data stream at time $t$ arrives to learner $i$ with context information $x_i(t)$. The context may be generated as a result of pre-classification or a header of the data stream.
%
For simplicity we assume that the context space is ${\cal X} = [0,1]^d$, while our results will hold for any bounded $d$ dimensional context space. We also note that the data input is \cem{high} dimensional and its dimension is greater than $d$ (in most of the cases its much larger than $d$) \com{Give example of the dimension of the data in a network security data}. \cem{For example, the network security data we use in numerical results section has 42 features, while the dimension of the context we use is at most 1.}
\armv{In such a setting, exploiting the context information may significantly improve the classification accuracy while decreasing the classification cost. However, the rate of learning increases with the dimension of the context space, which results in a tradeoff between the rate of learning and the classification accuracy.
Exploiting the context information not only improves the classification accuracy but it can also decrease the classification cost since the context can also provide information about what features to extract from the data. 
}
\rmv{\cem{For example, in a network security application, the context can be the reputation of the network in which the data originates.
Based on the reputation, a learner may monitor the data stream originating from the network periodically, and this period can decrease with the reputation. This implies that the monitoring cost increases when the reputation of the network decreases.
}
}

Each classification function $k \in \cup_{i \in {\cal M}}{\cal F}_i$ has an unknown accuracy $\pi_k(x) \in [0,1]$, depending on the context $x$. The accuracy $\pi_k(x)$ represents the probability that 
an input stream with context $x$ will be labeled correctly when classification function $k$ is used to label it. \rev{Different classification functions can have different accuracies for the same context. Although we do not make any assumptions about the classification accuracy $\pi_k(x)$ and the classification cost $d_k$, in general one can assume that classification accuracy increases with classification cost (e.g., classification functions with higher resolution, better processing).
\armv{In this paper the cost $d_k$ is a generic term that can represent any known cost such as processing cost, delay cost, communication cost, etc. }
} 
\rev{We assume that each classifier has similar accuracies for similar contexts; we formalize this in terms of a (uniform) Lipschitz condition.}
\begin{assumption} \label{ass:lipschitz2}
For each $k \in \cup_{i \in {\cal M}} {\cal F}_i$, there exists $L>0$, $\alpha>0$ such that for all $x,x' \in {\cal X}$, we have
$|\pi_k(x) - \pi_k(x')| \leq L ||x-x'||^\alpha$,
where $||.||$ denotes the Euclidian norm in $\mathbb{R}^d$.
\end{assumption}

Assumption \ref{ass:lipschitz2} indicates that the accuracy of a classification function for similar contexts will be similar to each other. \rev{Even though the Lipschitz condition can hold with different constants $L_k$ and $\alpha_k$ for each classification function, taking $L$ to be the largest among $L_k$ and $\alpha$ to be the smallest among $\alpha_k$ we get the condition in Assumption \ref{ass:lipschitz2}.}
For example, the context can be the time of the day or/and the location from which the data originates.
Therefore, the relation between the classification accuracy and time can be written down as a Lipschitz condition. We assume that $\alpha$ is known by the learners. In the bandit setting this is referred to as {\em similarity information} \cite{slivkins2009contextual}, \cite{ortner2010exploiting}.
\armv{
Different from these works, \cem{we do not require $L$ to be known by the learners. Prior work cited above required techniques to estimate $L$ in order for their algorithms to work, which usually causes poor short time performance and order of regret that is slightly higher than the optimal order with known $L$. In our work $L$ appears only in the performance bounds we prove.}
}

The goal of learner $i$ is to explore the alternatives in ${\cal K}_i$ to learn the accuracies, while at the same time exploiting the best alternative for the context $x_t$ arriving at each time step $t$ that balances the accuracy and cost to minimize its long term loss due to uncertainty. Learner $i$'s problem can be modeled as a contextual bandit problem \cite{slivkins2009contextual, dudik2011efficient, langford2007epoch, chu2011contextual}. After labeling the input at time $t$, each learner observes the true label and updates the sample mean accuracy of the selected arm based on this. 
Accuracies translate into rewards in bandit problems.
In the next subsection we formally define the benchmark solution which is computed using perfect knowledge about classification accuracies. Then, we define the regret which is the performance loss due to uncertainty about classification accuracies.

\add{\vspace{-0.2in}}
\subsection{Optimal Classification with Complete Information} \label{sec:centralized}

Our benchmark when evaluating the performance of the learning algorithm is the optimal solution which selects the classification function $k$ with the highest accuracy minus cost for learner $i$ from the set $\cup_{j \in {\cal M}} {\cal F}_j$ given context $x_t$ at time $t$. We assume that the costs are normalized so the tradeoff between accuracy and cost is captured without using weights. Specifically, the optimal solution we compare against is given by
\add{\vspace{-0.2in}}
\begin{align}
k^*(x) = \argmax_{k \in {\cal K}_i} \pi_k(x) - d_k, ~~ \forall x \in {\cal X}. \label{eqn:opt2}
\end{align}
%

%
%
%
%
%
%
Knowing the optimal solution means that learner $i$ knows the classification function in $\cup_{i \in {\cal M}} {\cal F}_i$ that yields the \rev{highest expected accuracy} for each $x \in {\cal X}$. Choosing the best classification function for each context $x$ requires to evaluate the accuracy minus cost for each context and is computationally intractable, because the context space ${\cal X}$ has infinitely many elements. 
\rmv{
Note that the problem remains hard even if we would put additional structure on the optimal classification scheme. For instance, we could assume that the 
%
optimal classification scheme for learner $i$ partitions ${\cal X}$ into  $|\cup_{i \in {\cal M}} {\cal F}_i|$ sets in each of which a single classification function is optimal.
%
Assume for an instance that the data arrival process to learner $i$ is i.i.d. with density $q_i$ and learner $i$ has access to all classification functions $\cup_{i \in {\cal M}} {\cal F}_i$. If learner $i$ knows $q_i$ and the classification accuracies of all classification functions in $\cup_{i \in {\cal M}} {\cal F}_i$, then 
under the above assumption, learner $i$ could compute the optimal classification regions
\begin{align*}
{\cal R}^* = \{R^*_k\}_{(k \in \cup_{i \in {\cal M}} {\cal F}_i)},
\end{align*}
of ${\cal X}$, by solving
\begin{align}
\textrm{{\bf (P1)  }} {\cal R}^* = \argmin_{{\cal R} \in \Theta} E \left[ \left| Y - \sum_{k \in \cup_{i \in {\cal M}} {\cal F}_i} k(X) I(X \in R_l) \right| + \sum_{k \in {\cal F}_i} d_k I(X \in R_k) + \sum_{k \in {\cal M}_{-i}}  d_k I \left(X \in \cup_{j \in {\cal F}_k} R_j \right)   \right], \label{eqn:opt}
\end{align}
where $\Theta$ is the set of $|\cup_{i \in {\cal M}} {\cal F}_i|$-set partitions of ${\cal X}$, the expectation is taken with respect to the distribution $q_i$ and $Y$ is the random variable denoting the true label. Here $I(X \in R_l)$ denotes the event that data received by the learner belongs to the $l$-th set of the partition ${\cal R}$ of ${\cal X}$.
The complexity of finding the optimal classification regions increases exponentially with $|\cup_{i \in {\cal M}} {\cal F}_i|$.
Importantly, note that the learning problem we are trying to solve is even harder than this because the learners are distributed and thus, each learner cannot directly access to all classification functions (learner $i$ only knows set ${\cal F}_i$ and ${\cal M}$, but it does not know any ${\cal F}_j, j \in {\cal M}_{-i}$), and the distributions $q_i$, $i \in {\cal M}$ are unknown (they need not to be i.i.d. or Markovian). Therefore, we use online learning techniques that do not rely on solving the optimization problem in (\ref{eqn:opt}) nor on an estimated version of it. 

\cem{Note that we discussed the above assumption only to illustrate that the optimal solution is still computationally
hard even when we put extra structure on the problem. Our results  
only require Assumption \ref{ass:lipschitz2} to hold.}
}
 
%
%
\add{\vspace{-0.2in}}
\subsection{The Regret of Learning}

In this subsection we define the regret as a performance measure of the learning algorithm used by the learners. Simply, the regret is the loss incurred due to the unknown system dynamics. Regret of a learning algorithm for learner $i$ is defined with respect to the best arm $k^*(x)$ given in (\ref{eqn:opt2}).
The regret of a learning algorithm is given by
\add{\vspace{-0.05in}}
\begin{align*}
R(T) &:= \sum_{t=1}^T \pi_{k^*(x_t)}(x_t) 
- E \left[ \sum_{t=1}^T ( I(k(x_t) = y_t) - d_{k(x_t)}) \right] ,
\end{align*}
where $k(x_t)$ denotes the classification function or other learner chosen at time $t$, $y_t$ denotes the true label and the expectation is taken with respect to the random selection made by the learning algorithm. Regret gives the convergence rate of the total expected reward of the learning algorithm to the value of the optimal solution given in (\ref{eqn:opt2}). Any algorithm whose regret is sublinear, i.e., $R(T) = O(T^\gamma)$ such that $\gamma<1$, will converge to the optimal solution in terms of the average reward. 
In the following section we will propose a distributed learning algorithm with sublinear regret.

\armv{In the next section, we propose a learning algorithm which divides the context space $[0,1]^d$ to $(m_T)^d$ hypercubes, and estimates the best classification function or best learner to call, in each of these hypercubes. Here, the number $m_T$ depends on the time horizon $T$ and is nondecreasing in $T$ which means that the number of hypercubes we consider is nondecreasing in $T$. The longer the time horizon, the finer should the partitions be in order to control the suboptimality resulting from taking averages over the entire hypercube. Secondly, we propose a distributed zooming algorithm that adaptively adjusts the number of hypercubes by zooming into the regions of the context space with high context arrival intensity.}

\vspace{-0.15in}
\section{A distributed uniform context partitioning algorithm} \label{sec:iid}

In this section we consider a uniform partitioning algorithm.
Assume that each learner runs the learning algorithm {\em Classify or Send for Classification} (CoS) given in Figure \ref{fig:COS}. Let $m_T$ be the {\em slicing parameter} of CoS that determines the partition of the context space ${\cal X}$. Basically, choosing a large $m_T$ will improve classification accuracy while increasing the number of explorations. We will analyze the performance for a fixed $m_T$ and then optimize over it. 
CoS forms a partition of $[0,1]^d$ consisting of $(m_T)^d$ sets where each set is a $d$-dimensional hypercube with dimensions $1/m_T \times 1/m_T \times \ldots \times 1/m_T$. Let ${\cal P}_T = \{P_1, P_2, \ldots, P_{(m_T)^d} \}$ denote this partition where each $P_l$ is a hypercube. When clear from the context, we will use $l$ instead of $P_l$ to denote the hypercube.
The set of arms for learner $i$ consists of its classification functions and the set of learners it can send the data to, which is denoted by ${\cal K}_i$.

%
%


\begin{figure}[htb]
\add{\vspace{-0.1in}}
\fbox {
\begin{minipage}{0.95\columnwidth}
{\fontsize{8}{7}\selectfont
\flushleft{Classify or Send for Classification (CoS for learner $i$):}
\begin{algorithmic}[1]
\STATE{Input: $D_1(t)$, $D_2(t)$, $D_3(t)$, $T$, $m_T$}
\STATE{Initialize: Partition $[0,1]^d$ into $(m_T)^d$ sets. Let ${\cal P}_T = \{ P_1, \ldots, P_{(m_T)^d} \}$ denote the sets in this partition. $N^i_{k,l}=0, \forall k \in {\cal K}_i, P_l \in {\cal P}_T$, $N^i_{1,k,l}=0, \forall k \in {\cal M}_{-i}, P_l \in {\cal P}_T$.}
\WHILE{$t \geq 1$}
\FOR{$l=1,\ldots,(m_T)^d$}
\IF{$x_i(t) \in P_l$}
\IF{$\exists k \in {\cal F}_i \textrm{ such that } N^i_{k,l} \leq D_1(t)$}
\STATE{Run {\bf Explore}($k$, $N^i_{k,l}$, $\bar{r}_{k,l}$)}
\ELSIF{$\exists k \in {\cal M}_{-i} \textrm{ such that } N^i_{1,k,l} \leq D_2(t)$}
\STATE{Obtain $N^k_l(t)$ from $k$, set $N^i_{1,k,l} = N^k_l(t) - N^i_{k,l}$}
\IF{$N^i_{1,k,l} \leq D_2(t)$}
\STATE{Run {\bf Train}($k$, $N^i_{1,k,l}$)}
\ELSE
\STATE{Go to line 15}
\ENDIF
\ELSIF{$\exists k \in {\cal M}_{-i} \textrm{ such that } N^i_{k,l} \leq D_3(t)$}
\STATE{Run {\bf Explore}($k$, $N^i_{k,l}$, $\bar{r}_{k,l}$)}
\ELSE
\STATE{Run {\bf Exploit}($\boldsymbol{M}^i_l$, $\bar{\boldsymbol{r}}_l$, ${\cal K}_i$)}
\ENDIF
\ENDIF
\ENDFOR
\STATE{$t=t+1$}
\ENDWHILE
\end{algorithmic}
}
\end{minipage}
} \caption{Pseudocode for the CoS algorithm} \label{fig:COS}
\add{\vspace{-0.22in}}
\end{figure}
\begin{figure}[htb]
\fbox {
\begin{minipage}{0.95\columnwidth}
{\fontsize{8}{7}\selectfont
{\bf Train}($k$, $n$):
\begin{algorithmic}[1]
\STATE{select arm $k$}
\STATE{Receive reward $r_k(t) = I(k(x_i(t)) = y_t) - d_{k(x_i(t))}$}
\STATE{$n++$}
\end{algorithmic}
{\bf Explore}($k$, $n$, $r$):
\begin{algorithmic}[1]
\STATE{select arm $k$}
\STATE{Receive reward $r_k(t) = I(k(x_i(t)) = y_t) - d_{k(x_i(t))}$}
\STATE{$r = \frac{n r + r_k(t)}{n + 1}$}
\STATE{$n++$}
\end{algorithmic}
{\bf Exploit}($\boldsymbol{n}$, $\boldsymbol{r}$, ${\cal K}_i$):
\begin{algorithmic}[1]
\STATE{select arm $k\in \argmax_{j \in {\cal K}_i} r_j$}
\STATE{Receive reward $r_k(t) = I(k(x_i(t)) = y_t) - d_{k(x_i(t))}$}
\STATE{$\bar{r}_{k} = \frac{n_k \bar{r}_{k} + r_k(t)}{n_k + 1}$}
\STATE{$n_k++$}
\end{algorithmic}
}
\end{minipage}
} \caption{Pseudocode of the training, exploration and exploitation modules} \label{fig:mtrain}
\add{\vspace{-0.2in}}
\end{figure}

\comment{
\begin{figure}[htb]
\fbox {
\begin{minipage}{0.95\columnwidth}
{\fontsize{9}{9}\selectfont
{\bf Explore}($k$, $n$, $r$):
\begin{algorithmic}[1]
\STATE{select arm $k$}
\STATE{Receive reward $r_k(t) = I(k(x_i(t)) = y_t) - d_{k(x_i(t))}$}
\STATE{$r = \frac{n r + r_k(t)}{n + 1}$}
\STATE{$n++$}
\end{algorithmic}
}
\end{minipage}
} \caption{Pseudocode of the exploration module} \label{fig:mexplore}
\end{figure}
}

\comment{
\begin{figure}[htb]
\fbox {
\begin{minipage}{0.9\columnwidth}
{\fontsize{9}{9}\selectfont
{\bf Exploit}($\boldsymbol{n}$, $\boldsymbol{r}$, ${\cal K}_i$):
\begin{algorithmic}[1]
\STATE{select arm $k\in \argmax_{j \in {\cal K}_i} r_j$}
\STATE{Receive reward $r_k(t) = I(k(x_i(t)) = y_t) - d_{k(x_i(t))}$}
\STATE{$\bar{r}_{k} = \frac{n_k \bar{r}_{k} + r_k(t)}{n_k + 1}$}
\STATE{$n_k++$}
\end{algorithmic}
}
\end{minipage}
} \caption{Pseudocode of the exploitation module} \label{fig:mexploit}
\end{figure}
}
For each set in the partition ${\cal P}_T$, learner $i$ keeps several counters for each arm in ${\cal K}_i$. Any time step $t$ can be in one of the three phases: {\em training} phase in which learner $i$ trains another learner by sending its own data, {\em exploration} phase in which learner $i$ updates the estimated reward of an arm in ${\cal K}_i$ by selecting it, and {\em exploitation} phase in which learner $i$ selects the arm with the highest estimated reward. The pseudocodes of these phases are given in Figure \ref{fig:mtrain}.
Upon each data arrival, learner $i$ first checks to which set in the partition ${\cal P}_T$ the context belongs. 
Let $N^i_l(t)$ be the number of data arrivals in $P_l$ of learner $i$ by time $t$. 
For $k \in {\cal F}_i$, let $N^i_{k,l}(t)$ be the number of times arm $k$ is selected in response to a data arriving to set $P_l$ in the partition ${\cal P}_T$ by learner $i$ by time $t$.
Note that learner $i$ does not know anything about learner $k$'s classification functions.
Therefore, before forming estimates about the reward of $k$, it needs to make sure that $k$ will almost always select its optimal classification function when called by $i$. This is why the training phase is needed for learners $k \in {\cal M}_{-i}$. To separate training, exploration and exploitation phases, learner $i$ keeps two counters for $k \in {\cal M}_{-i}$.
The first one, i.e., $N^i_{1,k,l}(t)$, counts the number of data arrivals to learner $k$ in set $l$ by time $t$ which includes data arrivals with context $x_k(t') \in P_l$, $t' <t$ and data arrivals from learner $i$ to $k$ in the training phases of $i$.
The second one, i.e., $N^i_{2,k,l}(t)$, counts the number of data arrivals to learner $k$ that are used in $i$'s reward estimation of $k$. This is the number of times data is sent from learner $i$ to learner $k$ in the exploration phase or exploitation phase of learner $i$. For simplicity of notation we let $N^i_{k,l}(t) := N^i_{2,k,l}(t)$ for $k \in {\cal M}_{-i}$.
The values of these counters are random variables when the context arrival process is stochastic.
Based on the values of these counters at time $t$, learner $i$ either trains, explores or exploits an arm in ${\cal K}_i$. This three phase learning structure is one of the major components of our learning algorithm which makes it different than the algorithms proposed for the contextual bandits in the literature which only have exploration and exploitation phases.

When a context $x_i(t) \in P_l$ arrives, in order to make sure that all classification functions of all learners are explored sufficiently, learner $i$ checks if the following set is nonempty. 
\add{\vspace{-0.25in}}
\begin{align*}
{\cal S}_{i,l} := &\left\{ k \in {\cal F}_i \textrm{ such that } N^i_{k,l}(t) \leq D_1(t)  \textrm{ or } k \in {\cal M}_{-i} \right. \\
&\left. \textrm{ such that } N^i_{1,k,l}(t) \leq D_2(t) \textrm{ or } N^i_{2,k,l}(t) \leq D_3(t)   \right\}.
\end{align*} 
For $k \in {\cal M}_{-i}$, let ${\cal E}^i_{k,l}(t)$ be the set of rewards collected from selections of arm $k$ in set $l$ by time $t$ for which $N^i_{1,k,l}(t) > D_2(t)$. We note that, learner $i$ does not have to communicate with learner $k$ at each time step to update $N^i_{1,k,l}(t)$. It only needs to communicate when $N^i_{1,k,l}(t) \leq D_2(t)$. To obtain the correct value of $N^i_{1,k,l}(t)$ it needs to know $N^k_l(t)$, since $N^i_{1,k,l}(t)=N^k_l(t) - N^i_{2,k,l}(t)$. 
For $k \in {\cal F}_i$, let ${\cal E}^i_{k,l}(t)$ the set of rewards collected from arm $k$ by time $t$.
If ${\cal S}_{i,l} \neq \emptyset$, then learner $i$ explores by choosing randomly an arm $\alpha(t) \in {\cal S}_{i,l}$. If ${\cal S}_{i,l} = \emptyset$, this implies that all classification functions have been explored sufficiently, so that learner $i$ exploits by choosing the arm with the highest sample mean estimate, i.e.,
\add{\vspace{-0.1in}}
\begin{align}
\alpha(t) \in \argmax_{k \in {\cal K}_i} \bar{r}^i_{k,l}(t), \label{eqn:maximizer}
\end{align}
where $\bar{r}^i_{k,l}(t)$ is the sample mean of the rewards in set ${\cal E}^i_{k,l}(t)$. Explicitly,
%
$\bar{r}^i_{k,l}(t) = (\sum_{r \in {\cal E}^i_{k,l}(t)} r)/|{\cal E}^i_{k,l}(t)|$,
%
where each $r \in {\cal E}^i_{k,l}(t)$ is equal to $1-d_k$ if the classification is correct and $-d_k$ if the classification is wrong.
When there is more than one maximizer of (\ref{eqn:maximizer}), one of them is randomly selected. The exploration control functions $D_1(t)$, $D_2(t)$ and $D_3(t)$ ensure that each classification function is selected sufficiently many number of times so that the sample mean estimates $\bar{r}^i_{k,l}(t)$ are accurate enough. In the following subsection we prove an upper bound on the regret of CoS.
\vspace{-0.2in}
\subsection{Analysis of the regret of CoS}

Let $\mu_{k}(x) = \pi_k(x) - d_k$, and $\beta_a = \sum_{t=1}^{\infty} 1/t^a$.
For each $P_l \in {\cal P}_T$ let 
%
$\overline{\mu}_{k,l} := \sup_{x \in P_l} \mu_k(x)$ and
$\underline{\mu}_{k,l} := \inf_{x \in P_l} \mu_k(x)$.
%
Let $x^*_l$ be the context at the center of the hypercube $P_l$. We define the optimal arm for $P_l$ as
%
$k^*(l) := \argmax_{k \in {\cal K}_i} \mu_k(x^*_l)$.
%
Let
\add{\vspace{-0.1in}}
\begin{align*}
{\cal L}^i_\theta(t) := \left\{ k \in {\cal K}_i \textrm{ such that }  \underline{\mu}_{k^*(l),l} - \overline{\mu}_{k,l} > a_1 t^{\theta} \right\},
\end{align*}
be the set of suboptimal arms for learner $i$ at time $t$, where 
$\theta<0$, $a_1 > 0$. The learners are not required to know the values of the parameters $\theta$ and $a_1$. They are only used in our analysis of the regret. First, we will give regret bounds that depend on values of $\theta$ and $a_1$ and then we will optimize over these values to find the best bound.

The regret given in (\ref{eqn:opt2}) can be written as a sum of three components: \add{$R(T) = E[R_e(T)] + E[R_s(T)] + E[R_n(T)]$,}
\remove{
\begin{align*}
R(T) = E[R_e(T)] + E[R_s(T)] + E[R_n(T)],
\end{align*}
}
where $R_e(T)$ is the regret due to training and explorations by time $T$, $R_s(T)$ is the regret due to suboptimal arm selections in exploitations by time $T$ and $R_n(T)$ is the regret due to near optimal arm selections in exploitations by time $T$, which are all random variables. In the following lemmas we will bound each of these terms separately. The following lemma bounds $E[R_e(T)]$. \rev{Due to space constraints some of the proofs are not included in the paper. For the complete proofs please see the online appendix \cite{}}.
\begin{lemma} \label{lemma:explorations}
When CoS is run by learner $i$ with parameters $D_1(t) = t^{z} \log t$, $D_2(t) = F_{\max} t^{z} \log t$, $D_3(t) = t^{z} \log t$ and $m_T = \left\lceil T^{\gamma} \right\rceil$\footnote{For a number $r \in \mathbb{R}$, let $\lceil r  \rceil$ be the smallest integer that is greater than or equal to $r$.}, where $0<z<1$ and $0<\gamma<1/d$, we have
\add{\vspace{-0.1in}}
\begin{align*}
E[R_e(T)] &\leq  \sum_{l=1}^{(m_T)^d} (|{\cal F}_i| + (M-1) (F_{\max} + 1)) T^{z} \log T \\
&+  (|{\cal F}_i| + 2(M-1)) (m_T)^d \\
&\leq 2^d (|{\cal F}_i| + (M-1) (F_{\max} + 1)) T^{z+\gamma d} \log T \\
&+ 2^d (|{\cal F}_i| + 2(M-1)) T^{\gamma d} ~.
\end{align*}
\end{lemma}
\remove{
\begin{proof}
Since time step $t$ is a training or an exploration step if and only if ${\cal S}_{i,l}(t) \neq \emptyset$, up to time $T$, there can be at most   $\left\lceil T^{z} \log T \right\rceil$ exploration steps in which a classification function in $k \in {\cal F}_i$ is selected by learner $i$, 
$\left\lceil F_{\max} T^{z} \log T \right\rceil$ training steps in which learner $i$ selects learner $k \in {\cal M}_{-i}$, $\left\lceil T^{z} \log T \right\rceil$ exploration steps in which learner $i$ selects learner $k \in {\cal M}_{-i}$. Result follows from summing these terms and the fact that $(m_T)^d \leq 2^d T^{\gamma d}$ for any $T \geq 1$.
\end{proof}
}
\rev{From Lemma \ref{lemma:explorations}, we see that the regret due to explorations is linear in the number of hypercubes $(m_T)^d$, hence exponential in parameter $\gamma$ and $z$. We conclude that $z$ and $\gamma$ should be small enough to achieve sublinear regret in exploration steps.
}
For any $k \in {\cal K}_i$ and $P_l \in {\cal P}_T$, the sample mean $\bar{r}_{k,l}(t)$ represents a random variable which is the average of the independent samples in set ${\cal E}^i_{k,l}(t)$. Different from classical finite-armed bandit theory \cite{auer}, these samples are not identically distributed. In order to facilitate our analysis of the regret, we generate two different artificial i.i.d. processes to bound the probabilities related to  $\bar{r}_{k,l}(t)$, $k \in {\cal K}_i$. The first one is the {\em best} process in which rewards are generated according to a bounded i.i.d. process with expected reward $\overline{\mu}_{k,l}$, the other one is the {\em worst} process in which the rewards are generated according to a bounded i.i.d. process with expected reward $\underline{\mu}_{k,l}$. Let $r^{\textrm{best}}_{k,l}(z)$ denote the sample mean of the $z$ samples from the best process and $r^{\textrm{worst}}_{k,l}(z)$ denote the sample mean of the $z$ samples from the worst process. We will bound the terms $E[R_n(T)]$ and $E[R_s(T)]$ by using these artificial processes along with the similarity information given in Assumption \ref{ass:lipschitz2}.
The following lemma bounds $E[R_s(T)]$.
\begin{lemma} \label{lemma:suboptimal1}
When CoS is run with parameters $D_1(t) = t^{z} \log t$, $D_2(t) = F_{\max} t^{z} \log t$, $D_3(t) = t^{z} \log t$ and $m_T = \left\lceil T^{\gamma} \right\rceil$, where $0<z<1$ and $0<\gamma<1/d$, given that
\begin{align*}
&2 L( \sqrt{d})^\alpha t^{- \gamma \alpha} + 6 t^{-z/2} \leq a_1 t^\theta, 
\end{align*}
we have
\vspace{-0.1in}
\begin{align*}
E[R_s(T)] &\leq   2^{d+1} (M-1+|{\cal F}_i|) \beta_2 T^{\gamma d} \\
&+ \frac{2^{d+2} (M-1)  F_{\max} \beta_2}{z} T^{\gamma d + z/2}
\end{align*}
\end{lemma}
\remove{
\begin{proof}
Let $\Omega$ denote the space of all possible outcomes, and $w$ be a sample path. The event that the algorithm exploits at time $t$ is given by
\begin{align*}
{\cal W}^i_{l}(t) := \{ w : S_{i,l}(t) = \emptyset  \}.
\end{align*}
We will bound the probability that the algorithm chooses a suboptimal arm in an exploitation step. Using that we can bound the expected number of times a suboptimal arm is chosen by the algorithm. Note that every time a suboptimal arm is chosen, since $\pi_k(x) - d_k \in [-1,1]$, the loss is at most $2$. Therefore $2$ times the expected number of times a suboptimal arm is chosen in an exploitation step bounds the regret due to suboptimal choices in exploitation steps.
Let ${\cal V}^i_{k,l}(t)$ be the event that a suboptimal action $k$ is chosen at time $t$. We have
\begin{align*}
R_s(T) &\leq \sum_{l \in {\cal P}_T} \sum_{t=1}^T \sum_{k \in {\cal L}^i_\theta(t)} I({\cal V}^i_{k,l}(t), {\cal W}^i_{l}(t) ).
\end{align*}
Taking the expectation
\begin{align}
E[R_s(T)] \leq \sum_{l \in {\cal P}_T} \sum_{t=1}^T \sum_{k \in {\cal L}^i_\theta(t)} P({\cal V}^i_{k,l}(t), {\cal W}^i_{l}(t) ) \label{eqn:subregret}
\end{align}

Let ${\cal B}^i_{k,l}(t)$ be the event that at most $t^{\phi}$ samples in ${\cal E}^i_{k,l}(t)$ are collected from suboptimal classification functions of the $k$-th arm. Obviously for any $k \in {\cal F}_i$, ${\cal B}^i_{k,l}(t) = \Omega$, while this is not always true for $k \in {\cal M}_{-i}$. 
We have
\begin{align}
\{ {\cal V}^i_{k,l}(t), {\cal W}^i_{l}(t)\} &\subset \left\{ \bar{r}_{k,l}(t) \geq \bar{r}_{k^*(l),l}(t), {\cal W}^i_{l}(t), {\cal B}^i_{k,l}(t) \right\} 
\cup \left\{ \bar{r}_{k,l}(t) \geq \bar{r}_{k^*(l),l}(t), {\cal W}^i_{l}(t), {\cal B}^i_{k,l}(t)^c \right\} \notag \\
&\subset \left\{ \bar{r}_{k,l}(t) \geq \overline{\mu}_{k,l} + H_t, {\cal W}^i_{l}(t), {\cal B}^i_{k,l}(t)  \right\} 
\cup \left\{ \bar{r}_{k^*(l),l}(t) \leq \underline{\mu}_{k^*(l),l} - H_t, {\cal W}^i_{l}(t), {\cal B}^i_{k,l}(t)  \right\} \notag \\
& \cup \left\{ \bar{r}_{k,l}(t) \geq \bar{r}_{k^*(l),l}(t), 
\bar{r}_{k,l}(t) < \overline{\mu}_{k,l} + H_t,
 \bar{r}_{k^*(l),l}(t) > \underline{\mu}_{k^*(l),l} - H_t,
{\cal W}^i_{l}(t) ,{\cal B}^i_{k,l}(t)  \right\} \notag \\
&\cup {\cal B}^i_{k,l}(t)^c , \label{eqn:vkt}
\end{align}
for some $H_t >0$. This implies that 
\begin{align}
P \left( {\cal V}^i_{k,l}(t), {\cal W}^i_{l}(t) \right) 
&\leq P \left( \bar{r}_{k,l}(t) \geq \overline{\mu}_{k,l} + H_t, {\cal W}^i_{l}(t), {\cal B}^i_{k,l}(t)  \right) \notag  \\
&+ P \left( \bar{r}_{k^*(l),l}(t) \leq \underline{\mu}_{k^*(l),l} - H_t, {\cal W}^i_{l}(t), {\cal B}^i_{k,l}(t) \right) \notag  \\
&+ P \left( \bar{r}_{k,l}(t) \geq \bar{r}_{k^*(l),l}(t), 
\bar{r}_{k,l}(t) < \overline{\mu}_{k,l} + H_t,
 \bar{r}_{k^*(l),l}(t) > \underline{\mu}_{k^*(l),l} - H_t,
{\cal W}^i_{l}(t), {\cal B}^i_{k,l}(t)  \right) \notag \\
&+ P({\cal B}^i_{k,l}(t)^c). \label{eqn:ubound1}
\end{align}
We have for any suboptimal arm $k \in {\cal K}_i$
\begin{align}
& P \left( \bar{r}_{k,l}(t) \geq \bar{r}_{k^*(l),l}(t), 
\bar{r}_{k,l}(t) < \overline{\mu}_{k,l} + H_t,
 \bar{r}_{k^*(l),l}(t) > \underline{\mu}_{k^*(l),l} - H_t,
{\cal W}^i_{l}(t), {\cal B}^i_{k,l}(t)  \right) \notag \\
&\leq P \left( \bar{r}^{\textrm{best}}_{k,l}(|{\cal E}^i_{k,l}(t)|) 
\geq \bar{r}^{\textrm{worst}}_{k^*(l),l}(|{\cal E}^i_{k^*(l),l}(t)|)
-  t^{\phi-1} , 
\bar{r}^{\textrm{best}}_{k,l}(|{\cal E}^i_{k,l}(t)|) < \overline{\mu}_{k,l} + L \left( \frac{\sqrt{d}}{m_T} \right)^\alpha + H_t +  t^{\phi-1}, \right. \notag \\
& \left. \bar{r}^{\textrm{worst}}_{k^*(l),l}(|{\cal E}^i_{k^*(l),l}(t)|) > \underline{\mu}_{k^*(l),l} - L \left( \frac{\sqrt{d}}{m_T} \right)^\alpha - H_t,
{\cal W}^i_{l}(t)    \right). \notag
\end{align}
Since $k$ is a suboptimal arm, when
\begin{align}
2 L \left( \frac{\sqrt{d}}{m_T} \right)^\alpha + 2H_t + 2t^{\phi-1} - a_1 t^\theta \leq 0,
\label{eqn:boundcond}
\end{align}
the three inequalities given below
\begin{align*}
& \underline{\mu}_{k^*(l),l} - \overline{\mu}_{k,l} > a_1 t^{\theta},\\
& \bar{r}^{\textrm{best}}_{k,l}(|{\cal E}^i_{k,l}(t)|) < \overline{\mu}_{k,l} + L \left( \frac{ \sqrt{d}}{m_T} \right)^\alpha + H_t + t^{\phi-1} ,\\
& \bar{r}^{\textrm{worst}}_{k^*(l),l}(|{\cal E}^i_{k,l}(t)|) > \underline{\mu}_{k^*(l),l} - L \left( \frac{ \sqrt{d}}{m_T} \right)^\alpha - H_t,
\end{align*}
together imply that 
\begin{align*}
\bar{r}^{\textrm{best}}_{k,l}(|{\cal E}^i_{k,l}(t)|) < \bar{r}^{\textrm{worst}}_{k^*(l),l}(|{\cal E}^i_{k,l}(t)|) -  t^{\phi-1} ,
\end{align*}
which implies that for a suboptimal arm $k \in {\cal K}_i$, we have
\begin{align}
P \left( \bar{r}_{k,l}(t) \geq \bar{r}_{k^*(l),l}(t), 
\bar{r}_{k,l}(t) < \overline{\mu}_{k,l} + H_t,
 \bar{r}_{k^*(l),l}(t) > \underline{\mu}_{k^*(l),l} - H_t,
{\cal W}^i_{l}(t), {\cal B}^i_{k,l}(t)  \right) = 0. \label{eqn:vktbound1}
\end{align}
Let $H_t = 2 t^{\phi-1}$. Then a sufficient condition that implies (\ref{eqn:boundcond}) is
\begin{align}
&2 L( \sqrt{d})^\alpha t^{- \gamma \alpha} + 6 t^{\phi-1} \leq a_1 t^\theta. \label{eqn:maincondition}
\end{align}
Assume that (\ref{eqn:maincondition}) holds for all $t \geq 1$.
Using a Chernoff-Hoeffding bound, for any $k \in {\cal L}^i_{\theta}(t)$, since on the event ${\cal W}^i_{l}(t)$, $|{\cal E}^i_{k,l}(t)| \geq t^z \log t$, we have
\begin{align}
P \left( \bar{r}_{k,l}(t) \geq \overline{\mu}_{k,l} + H_t, {\cal W}^i_{l}(t), {\cal B}^i_{k,l}(t) \right) 
&\leq P \left( \bar{r}^{\textrm{best}}_{k,l}(|{\cal E}^i_{k,l}(t)|) \geq \overline{\mu}_{k,l} + H_t, {\cal W}^i_{l}(t) \right) \notag \\
&\leq e^{-2 (H_t)^2 t^z \log t}  = e^{-8 t^{2\phi-2} t^z \log t} ~, \label{eqn:vktbound2}
\end{align}
and
\begin{align}
&P \left( \bar{r}_{k^*(l),l}(t) \leq \underline{\mu}_{k^*(l),l} - H_t, {\cal W}^i_{l}(t), {\cal B}^i_{k,l}(t) \right) \notag \\
&\leq P \left( \bar{r}^{\textrm{worst}}_{k^*(l),l}(|{\cal E}^i_{k^*(l),l}(t)|)  \leq \underline{\mu}_{k^*(l),l} - H_t +  t^{\phi-1}, {\cal W}^i_{l}(t) \right) \notag \\
&\leq e^{-2 (H_t -  t^{\phi-1})^2 t^z \log t} = e^{-2 t^{2\phi-2} t^z \log t}. \label{eqn:vktbound3}
\end{align}
In order to bound the regret, we will sum (\ref{eqn:vktbound2}) and (\ref{eqn:vktbound3}) for all $t$ up to $T$. For regret to be small we want the sum to be sublinear in $T$. This holds when $2\phi -2 +z \geq 0$. We want $z$ to be small since regret due to explorations increases with $z$, and we also want $\phi$ to be small since we will show that our regret bound increases with $\phi$. Therefore we set $2\phi -2 +z =0$, hence 
\begin{align}
\phi = 1-z/2. \label{eqn:maincondition2}
\end{align}
When (\ref{eqn:maincondition2}) holds we have
\begin{align}
P \left( \bar{r}_{k,l}(t) \geq \overline{\mu}_{k,l} + H_t, {\cal W}^i_{l}(t), {\cal B}^i_{k,l}(t) \right) \leq \frac{1}{t^2}, \label{eqn:vktbound22}
\end{align}
and
\begin{align}
P \left( \bar{r}_{k^*(l),l}(t) \leq \underline{\mu}_{k^*(l),l} - H_t, {\cal W}^i_{l}(t), {\cal B}^i_{k,l}(t) \right) \leq \frac{1}{t^2}. \label{eqn:vktbound32}
\end{align}

Finally, for $k \in {\cal F}_i$ obviously we have $P({\cal B}^i_{k,l}(t)^c)=0$. For $k \in {\cal M}_{-i}$, let $X^i_{k,l}(t)$ denote the random variable which is the number of times a suboptimal classification function for arm $k$ is chosen in exploitation steps when the context is in set $P_l$ by time $t$. We have $\{ {\cal B}^i_{k,l}(t)^c, {\cal W}^i_l(t)  \} = \{ X^i_{k,l}(t) \geq t^\phi \}$. Applying the Markov inequality we have
\begin{align*}
P({\cal B}^i_{k,l}(t)^c, {\cal W}^i_l(t)) \leq \frac{E[X^i_{k,l}(t)]}{t^\phi},
\end{align*}
Let $\Xi^i_{k,l}(t)$ be the event that a suboptimal classification function $m \in {\cal F}_k$ is called by learner $k \in {\cal M}_{-i}$, when it is invoked by learner $i$ for the $t$-th time in the exploitation phase of learner $i$. 
We have 
\begin{align*}
X^i_{k,l}(t) = \sum_{t'=1}^{{\cal E}^i_{k,l}(t)} I(\Xi^i_{k,l}(t')),
\end{align*}
and
\begin{align*}
P \left( \Xi^i_{k,l}(t) \right) 
&\leq \sum_{m \in {\cal L}^k_\theta} P \left( \bar{r}_{m,l}(t) \geq \bar{r}^{*k}_{l}(t) \right) \\
&\leq \sum_{m \in {\cal L}^k_\theta}
\left(  P \left( \bar{r}_{m,l}(t) \geq \overline{\mu}_{m,l} + H_t, {\cal W}^i_{l}(t) \right)   
+ P \left( \bar{r}^{*k}_{l}(t) \leq \underline{\mu}^{*k}_{l} - H_t, {\cal W}^i_{l}(t) \right) \right. \\
&\left. + P \left( \bar{r}_{m,l}(t) \geq \bar{r}^{*k}_{l}(t), 
\bar{r}_{m,l}(t) < \overline{\mu}_{m,l} + H_t,
 \bar{r}^{*k}_{l}(t) > \underline{\mu}^{*k}_{l} - H_t ,
{\cal W}^i_{l}(t) \right)  \right).
\end{align*}
When (\ref{eqn:maincondition}) holds, since $\phi = 1 - z/2$, the last probability in the sum above is equal to zero while the first two inequalities are upper bounded by $e^{-2(H_t)^2 t^z \log t}$. This is due to the second phase of the exploration algorithm which requires at least $t^z \log t$ samples from the second exploration phase for all learners before the algorithm exploits any learner. Therefore, we have
\begin{align*}
P \left( \Xi^i_{k,l}(t) \right) \leq \sum_{m \in {\cal L}^k_\theta} 2 e^{-2(H_t)^2 t^z \log t} \leq \frac{2 |{\cal F}_k|}{t^2}.
\end{align*}
These together imply that 
\begin{align*}
E[X^i_{k,l}(t)] \leq \sum_{t'=1}^{\infty} P(\Xi^i_{k,l}(t')) \leq 2 |{\cal F}_k| \sum_{t'=1}^\infty \frac{1}{t^2}.
\end{align*}
Therefore from the Markov inequality we get
\begin{align}
P({\cal B}^i_{k,l}(t)^c, {\cal W}^i_l(t)) = P(X^i_{k,l}(t) \geq t^\phi) \leq \frac{2 |{\cal F}_k| \beta_2}{t^{1-z/2}}. \label{eqn:selectionbound}
\end{align}
Then using (\ref{eqn:vktbound1}), (\ref{eqn:vktbound22}), (\ref{eqn:vktbound32}) and (\ref{eqn:selectionbound}) we have 
\begin{align*}
P \left( {\cal V}^i_{k,l}(t), {\cal W}^i_l(t)  \right) \leq \frac{2}{t^{2}} + \frac{2 |{\cal F}_k| \beta_2}{t^{1-z/2}},
\end{align*}
for any $k \in {\cal M}_{-i}$, and
\begin{align*}
P \left( {\cal V}^i_{k,l}(t), {\cal W}^i_l(t)  \right) \leq \frac{2}{t^{2}},
\end{align*}
for any $k \in {\cal F}_i$. By (\ref{eqn:subregret}), we have
\begin{align}
E[R_s(T)] &\leq 2^d T^{\gamma d}
 \left( 2 (M-1+|{\cal F}_i|) \beta_2  + 2 (M-1) F_{\max} \beta_2 \sum_{t=1}^T \frac{1}{t^{1-z/2}} \right) \\
&\leq  2^{d+1} (M-1+|{\cal F}_i|) \beta_2 T^{\gamma d}+ \frac{2^{d+2} (M-1)  F_{\max} \beta_2}{z} T^{\gamma d + z/2}, \label{eqn:regret_s}
\end{align}
where (\ref{eqn:regret_s}) follows from Appendix \ref{app:seriesbound}. 
\end{proof}
}
\rev{From Lemma \ref{lemma:suboptimal1}, we see that the regret increases exponentially with parameters $\gamma$ and $z$, similar to the result of Lemma \ref{lemma:explorations}. These two lemmas suggest that $\gamma$ and $z$ should be as small as possible, given the condition
\begin{align*}
&2 L( \sqrt{d})^\alpha t^{- \gamma \alpha} + 6 t^{-z/2} \leq a_1 t^\theta, 
\end{align*}
is satisfied. 
}

Each time learner $i$ selects another learner $k$ to label its data, learner $k$ calls one of its classification functions. There is a positive probability that learner $k$ will call one of its suboptimal classification functions, which implies that even if learner $k$ is near optimal for learner $i$, selecting learner $k$ may not yield a near optimal outcome. We need to take this into account, in order to bound $E[R_n(T)]$. 
\add{The following lemma gives the bound on $E[R_n(T)]$.}
\remove{
For $k \in {\cal M}_{-i}$, let $X^i_{k,l}(t)$ denote the random variable which is the number of times a suboptimal classification function for arm $k$ is chosen in exploitation steps when the context is in set $P_l$ by time $t$. The next lemma bounds the expected number of times a suboptimal classification functions is chosen when learner $i$ calls a near optimal learner.
\begin{lemma} \label{lemma:callother}
When CoS is run with parameters $D_1(t) = t^{z} \log t$, $D_2(t) = F_{\max} t^{z} \log t$, $D_3(t) = t^{z} \log t$ and $m_T = \left\lceil T^{\gamma} \right\rceil$, where $0<z<1$ and $0<\gamma<1/d$, given that
\begin{align*}
&2 L( \sqrt{d})^\alpha t^{- \gamma \alpha} + 6 t^{-z/2} \leq a_1 t^\theta, 
\end{align*}
we have
\begin{align*}
E[X^i_{k,l}(t)] \leq 2 F_{\max} \beta_2.
\end{align*}
\end{lemma}
\remove{
\begin{proof}
Let $\Xi^i_{k,l}(t)$ be the event that a suboptimal classification function $m \in {\cal F}_k$ is called by learner $k \in {\cal M}_{-i}$, when it is invoked by learner $i$ for the $t$-th time in the exploitation phase of learner $i$. 
We have 
\begin{align*}
X^i_{k,l}(t) = \sum_{t'=1}^{{\cal E}^i_{k,l}(t)} I(\Xi^i_{k,l}(t')),
\end{align*}
and
\begin{align*}
P \left( \Xi^i_{k,l}(t) \right) 
&\leq \sum_{m \in {\cal L}^k_\theta} P \left( \bar{r}_{m,l}(t) \geq \bar{r}^{*k}_{l}(t) \right) \\
&\leq \sum_{m \in {\cal L}^k_\theta}
\left(  P \left( \bar{r}_{m,l}(t) \geq \overline{\mu}_{m,l} + H_t, {\cal W}^i_{l}(t) \right)   
+ P \left( \bar{r}^{*k}_{l}(t) \leq \underline{\mu}^{*k}_{l} - H_t, {\cal W}^i_{l}(t) \right) \right. \\
&\left. + P \left( \bar{r}_{m,l}(t) \geq \bar{r}^{*k}_{l}(t), 
\bar{r}_{m,l}(t) < \overline{\mu}_{m,l} + H_t,
 \bar{r}^{*k}_{l}(t) > \underline{\mu}^{*k}_{l} - H_t ,
{\cal W}^i_{l}(t) \right)  \right).
\end{align*}
Let $H_t = 2 t^{-z/2}$. Similar to the proof of Lemma \ref{lemma:suboptimal1}, the last probability in the sum above is equal to zero while the first two inequalities are upper bounded by $e^{-2(H_t)^2 t^z \log t}$. This is due to the second phase of the exploration algorithm which requires at least $t^z \log t$ samples from the second exploration phase for all learners before the algorithm exploits any learner. Therefore, we have
\begin{align*}
P \left( \Xi^i_{k,l}(t) \right) \leq \sum_{m \in {\cal L}^k_\theta} 2 e^{-2(H_t)^2 t^z \log t} \leq \frac{2 |{\cal F}_k|}{t^2}.
\end{align*}
These together imply that 
\begin{align*}
E[X^i_{k,l}(t)] \leq \sum_{t'=1}^{\infty} P(\Xi^i_{k,l}(t')) \leq 2 |{\cal F}_k| \sum_{t'=1}^\infty \frac{1}{t^2}.
\end{align*}
\end{proof}
}

We will use Lemma \ref{lemma:callother} in the following lemma to bound $E[R_n(T)]$.
}
\begin{lemma} \label{lemma:nearoptimal}
When CoS is run with parameters $D_1(t) = t^{z} \log t$, $D_2(t) = F_{\max} t^{z} \log t$, $D_3(t) = t^{z} \log t$ and $m_T = \left\lceil T^{\gamma} \right\rceil$, where $0<z<1$ and $0<\gamma<1/d$, given that
\begin{align*}
&2 L( \sqrt{d})^\alpha t^{- \gamma \alpha} + 6 t^{-z/2} \leq a_1 t^\theta, 
\end{align*}
we have
\begin{align*}
E[R_n(T)] \leq \frac{2 a_1 T^{1+\theta}}{1+\theta} + 4(M-1)F_{\max} \beta_2.
\end{align*}
\end{lemma}
\remove{
\begin{proof}
If a near optimal arm in ${\cal F}_i$ is chosen at time $t$, the contribution to the regret is at most $a_1 t^{\theta}$. If a near optimal arm in $k \in {\cal M}_{-i}$ is chosen at time $t$, and if $k$ classifies according to one of its near optimal classification functions than the contribution to the regret is at most $2 a_1 t^{\theta}$.
Therefore the total regret due to near optimal arm selections in ${\cal K}_i$ by time $T$ is upper bounded by 
\begin{align*}
2 a_1 \sum_{t=1}^T t^{\theta} & \leq \frac{2 a_1 T^{1+\theta}}{1+\theta},
\end{align*}
by using the result in Appendix \ref{app:seriesbound}. 
Each time a near optimal arm in $k \in {\cal M}_{-i}$ is chosen in an exploitation step, there is a small probability that the classification function called by arm $k$ is a suboptimal one. Given in Lemma \ref{lemma:callother}, the expected number of times a suboptimal classification function is called is bounded by $2 |{\cal F}_k| \beta_2$. Each time a suboptimal classification function is chosen the regret can be at most $2$.
\end{proof}
}
\rev{From Lemma \ref{lemma:nearoptimal}, we see that the regret due to near optimal arms depends exponentially on $\theta$ which is related to negative of $\gamma$ and $z$. Therefore $\gamma$ and $z$ should be chosen as large as possible to minimize the regret due to near optimal arms.}


Combining the above lemmas, we obtain the finite time, uniform regret bound given in the following theorem.
\begin{theorem}\label{theorem:cos}
Let the CoS algorithm run with exploration control functions $D_1(t) = t^{2\alpha/(3\alpha+d)} \log t$, $D_2(t) = F_{\max} t^{2\alpha/(3\alpha+d)} \log t$, $D_3(t) = t^{2\alpha/(3\alpha+d)} \log t$ and slicing parameter $m_T = T^{1/(3\alpha + d)}$. Then,
\begin{align*}
E[R(T)] &\leq T^{\frac{2\alpha+d}{3\alpha+d}}
\left( \frac{2 (2 L d^{\alpha/2}+6)}{\frac{2\alpha+d}{3\alpha+d}} + 2^d Z_i \log T \right) \\
&+ T^{\frac{\alpha+d}{3\alpha+d}} \frac{2^{d+2} (M-1) F_{\max} \beta_2}{\frac{2\alpha}{3\alpha+d}} \\
&+ T^{\frac{d}{3\alpha+d}} 2^d (2 Z_i \beta_2 
+ |{\cal K}_i|) + 4 (M-1) F_{\max} \beta_2,
\end{align*}
where $Z_i = {\cal F}_i + (M-1)(F_{\max}+1)$.
\end{theorem}
\begin{proof}
The highest orders of regret come from explorations and near optimal arms which are $O(T^{\gamma d + z})$ and $O(T^{1+\theta})$ respectively. We need to optimize them with respect to the constraint
\add{\vspace{-0.1in}}
\begin{align*}
2 Ld^{\alpha/2} t^{- \gamma \alpha} + 6 t^{-z/2} \leq a_1 t^\theta, 
\vspace{-0.1in}
\end{align*}
which is assumed in Lemmas \ref{lemma:suboptimal1} and \ref{lemma:nearoptimal}.
The values that minimize the regret for which this constraint hold is $\theta = -z/2$, $\gamma = z/(2\alpha)$ $a_1 = 2 Ld^{\alpha/2} + 6$ and $z = 2\alpha/(3\alpha+d)$. 
Result follows from summing the bounds in Lemmas \ref{lemma:explorations}, \ref{lemma:suboptimal1} and \ref{lemma:nearoptimal}. 
\end{proof}
\begin{remark}
Although the parameter $m_T$ of CoS depends on $T$ and hence we require $T$ as an input to the algorithm, we can make CoS run independent of the final time $T$ and achieve the same regret bound by using a well known doubling trick (see, e.g., \cite{slivkins2009contextual}). Consider phases $\tau \in \{1,2,\ldots\}$, where each phase has length $2^{\tau}$. We run a new instance of algorithm CoS at the beginning of each phase with time parameter $2^\tau$. Then the regret of this algorithm up to any time $T$ will be $O\left(T^{(2\alpha+d)/(3\alpha+d)}\right)$. 
\end{remark}

The regret bound proved in Theorem \ref{theorem:cos} is sublinear in time which guarantees convergence in terms of the average reward, i.e., 
$\lim_{T \rightarrow \infty} E[R(T)]/T = 0$.
For example, when $\alpha=1$ and $d=1$, the order of the regret is $O\left(T^{3/4}\right)$. This implies that the convergence rate to the optimal average reward is $O\left(T^{-1/4}\right)$.
\armv{ In a network security application this means that after $10^8$ data samples, the error rate of CoS will be at most $0.01 \times c$ higher than the error rate of the optimal distributed classification scheme, for some constant $c>0$.}
For a fixed $\alpha$, the regret becomes linear in the limit as $d$ goes to infinity. On the contrary, when $d$ is fixed the regret decreases, and in the limit as $\alpha$ goes to infinity it becomes $O(T^{2/3})$. This is intuitive since increasing $d$ means that the dimension of the context increases and therefore the number of hypercubes to explore increases. While increasing \cem{$\alpha$ means that the level of similarity between any two pairs of contexts increases, i.e., knowing the accuracy of a classification function $k$ in one context yields more information about its accuracy in another context.}
\armv{
In our algorithm we used $d$ as an input parameter and compared with the optimal solution given a fixed $d$. However, the context information can also be adaptively chosen over time. 
For example, in network security, the context can be either time of the day, origin of the data or both. The classifier accuracies will depend on what is used as context information. The regret bound in Theorem \ref{theorem:cos} holds even if no data arrives to other classifiers. In general, the regret is much less than this bound. For example, consider the extreme case where the data arrival to each classifier is identical. This implies that classifier $i$ does not need to use the control function $D_2(t)$, since whenever classifier $i$ has sampled all of its own classification functions sufficiently many times, this will also be true for any other classifier $k \in {\cal M}_{-i}$. Therefore, training steps are not required in this scenario. 
}
%
\add{\vspace{-0.177in}}
\subsection{Regret of CoS for online learning classification functions}

In our analysis we assumed that given a context $x$, the classification function accuracy $\pi_k(x)$ is fixed. This holds when the classification functions are trained a priori, but the learners do not know the accuracy because $k$ is not tested yet. By using our contextual framework, we can also allow the classification functions to learn over time based on the data. \rev{Usually in Big Data applications we cannot have the classifiers being pre-trained as they are often deployed for the first time in a certain setting. For example in \cite{chai2002bayesian}, Bayesian online classifiers are used for text classification and filtering.}
%
We do this by introducing time as a context, thus increasing the context dimension to $d+1$. Time is normalized in interval $[0,1]$ such that $0$ corresponds to $t=0$, 1 corresponds to $t=T$ and each time slot is an interval of length $1/T$. 
For an online learning classifier, intuitively the accuracy is expected to increase with the number of samples, and thus, $\pi_k(x, t)$ will be non-decreasing in time. On the other hand, the increase in the accuracy in a single time step should be bounded. Otherwise, the online learning classifier will be able to classify all possible data streams without any error after a finite number of steps. Because of this, in addition to Assumption  \ref{ass:lipschitz2}, the following should hold for an online learning classifier:
\begin{align*}
\pi_k(x,(t+1)/T) \leq \pi_k(x,t/T) + L T^{-\alpha},
\end{align*}
for some $L$ and $\alpha$. Then we have the following theorem when online learning classifiers are present.
\begin{theorem}\label{theorem:cos2}
Let the CoS algorithm run with exploration control functions $D_1(t) = t^{2\alpha/(3\alpha+d+1)} \log t$, $D_2(t) = F_{\max} t^{2\alpha/(3\alpha+d+1)} \log t$, $D_3(t) = t^{2\alpha/(3\alpha+d+1)} \log t$ and slicing parameter $m_T = T^{1/(3\alpha + d+1)}$. Then,
\add{\vspace{-0.1in}}
\begin{align*}
E[R(T)] &\leq T^{\frac{2\alpha+d+1}{3\alpha+d+1}}
\left( \frac{2 (2 L (d+1)^{\alpha/2}+6)}{\frac{2\alpha+d+1}{3\alpha+d+1}} + 2^{d+1} Z_i \log T \right) \\
&+ T^{\frac{\alpha+d+1}{3\alpha+d+1}} \frac{2^{d+3} (M-1) F_{\max} \beta_2}{\frac{2\alpha}{3\alpha+d+1}} \\
&+ T^{\frac{d}{3\alpha+d+1}} 2^{d+1} (2 Z_i \beta_2 
+ |{\cal K}_i|) + 4 (M-1) F_{\max} \beta_2,
\end{align*}
where $Z_i = {\cal F}_i + (M-1)(F_{\max}+1)$.
\end{theorem} 
\armv{
\begin{proof}
The proof is the same as proof of Theorem \ref{theorem:cos}, with context dimension $d+1$ instead of $d$.
\end{proof}
}
The above theorem implies that the regret in the presence of classification functions that learn online based on the data is $O(T^{(2\alpha+d+1)/(3\alpha+d+1)})$. \rev{From the result of Theorem \ref{theorem:cos2}, we see that our notion of context can capture any relevant information that can be utilized to improve the classification.
\armv{Specifically, we showed that by treating time as one dimension of the context we can achieve sublinear regret bounds.
}
Compared to Theorem \ref{theorem:cos}, in Theorem \ref{theorem:cos2}, the exploration rate is reduced from $O(T^{2\alpha/(2\alpha+d)})$ to $O(T^{2\alpha/(2\alpha+d+1)})$, while the memory requirement is increased from $O(T^{d/(3\alpha+d)})$ to $O(T^{(d+1)/(3\alpha+d+1)})$.}

\add{\vspace{-0.2in}}
\subsection{Computational complexity of CoS}

For each set $P_l \in {\cal P}_T$, classifier $i$ keeps the sample mean of rewards from $|{\cal F}_i| + M -1$ arms, while for a standard centralized bandit algorithm, the sample mean of the rewards of $|\cup_{k \in {\cal M}} {\cal F}_k|$ arms needs to be kept in memory. Since the number of sets in ${\cal P}_T$ is upper bounded by $2^d T^{d/(3\alpha+d)} $, the memory requirement is upper bounded by
\add{\vspace{-0.1in}}
\begin{align}
(|{\cal F}_i| + M -1) 2^d T^{d/(3\alpha+d)}. \label{eqn:cosmemory}
\end{align}
\rev{This means that the memory requirement is sublinearly increasing in $T$ and thus, in the limit $T \rightarrow \infty$ goes to infinity, however, CoS can be modified so that the available memory provides an upper bound on $m_T$. However, in this case the regret bound given in Theorem \ref{theorem:cos} may not hold.} The following example illustrates that for a data set with a reasonable size, the memory requirement is not very high.
For example for $\alpha=1$, $d=1$, we have $2^d T^{d/(3\alpha+d)} = 2 T^{1/4}$. If classifier $i$ learned through $T=10^8$ samples, and if $M=100$, $|{\cal F}_k| =100$, for all $k \in {\cal M}$, CoS only need to store at most $40000$ sample mean estimates, while a standard bandit algorithm which does not exploit any context information requires to keep $10000$ sample mean estimates. Although, the memory requirement is $4$ times higher than the memory requirement of a standard bandit algorithm, CoS is suitable for a distributed implementation, and classifier $i$ does not require any knowledge about the classification functions of other classifiers (except an upper bound on the number of classification functions of other classifiers). Moreover, the regret of CoS is sublinear with respect to the best distributed classification scheme, while the regret of a standard bandit algorithm is only sublinear with respect to the best fixed classifier.

\armv{
\add{\vspace{-0.15in}}
\subsection{CoS with delayed feedback}

Next, we consider the case when the feedback is delayed. We assume that the label for data instance at time $t$ arrives with an $L_i(t)$ time slot delay, where $L_i(t)$ is a random variable such that $L_i(t) \leq L_{\max}$ with probability one for some $L_{\max}>0$ which is known to the algorithm. Algorithm CoS is modified so that it keeps in its memory the last $L_{\max}$ labels produced by classification and the indices are updated whenever a true label arrives. We have the following result for delayed label feedback.
\begin{corollary} \label{cor:uniform}
Consider the delayed feedback case where the true label of the data instance at time $t$ arrives at time $t+L_i(t)$, where $L_i(t)$ is a random variable with support in $\{0,1,\ldots, L_{\max}\}$, $L_{\max}>0$ is an integer. Let $R^{\textrm{nd}}(T)$ denote the regret of CoS with no delay by time $T$, and $R^{\textrm{d}}(T)$ denote the regret of modified CoS with delay by time $T$. Then we have,
%
$R^{\textrm{d}}(T) \leq L_{\max} + R^{\textrm{nd}}(T)$.
%
\end{corollary}
\remove{
\begin{proof}
By a Chernoff-Hoeffding bound, it can be shown that the probability of deviation of the sample mean accuracy from the true accuracy decays exponentially with the number of samples. A new sample is added to sample mean accuracy whenever the true label of a previous classification arrives. Note that the worst case is when all labels are delayed by $L_{\max}$ time steps. This is equivalent to starting the algorithm with an $L_{\max}$ delay. 
\end{proof}
}

The cost of label delay is additive which does not change the sublinear order of the regret. The memory requirement for CoS with no delay is $ |{\cal K}_i| (m_T)^d = 2^d (|{\cal F}_i| + M -1) T^{\frac{d}{ 3\alpha + d}}$, while memory requirement for CoS modified for delay is $L_{\max} + |{\cal K}_i| (m_T)^d$. Therefore, the order of memory cost is also independent of the delay. However, we see that the memory requirement significantly grows with the time horizon $T$ when the dimension of the context space $d$ is high, which makes CoS infeasible for learning from a very large set of samples. The algorithm we propose in the next section solves this issue by adaptively creating a partition of the context space which requires less than $(m_T)^d$ hypercubes.
}
\add{\vspace{-0.2in}}
\section{Numerical Results} \label{sec:numerical}

In this section we provide numerical results for CoS using network security data from KDD Cup 1999 data set. We compare the performance of our learning algorithm with {\em AdaBoost} \cite{freund1995desicion} and the online version of AdaBoost called {\em sliding window AdaBoost} \cite{oza2001online}. 

The network security data has 42 features. The goal is to predict at any given time if an attack occurs or not based on the values of the features.
We run the simulations for three different context information; (i) context is the label at the previous time step, (ii) context is the feature named {\em srcbytes}, which is the number of data bytes from source to destination, (iii) context is time. All the context information is normalized to be in $[0,1]$. There are $4$ local learners. Each local learner has $2$ classification functions. Classification costs $d_k$ is set to $0$ for all $k \in {\cal K}_1$.

All classification functions are trained using 5000 consecutive samples from different segments of the network security data. Then, they are tested on $T = 20000$ consecutive samples. We run simulations for two different sets of classifiers. 
In our first simulation S1, there are two good classifiers that have low number of errors on the test data, while in our second simulation S2, there are no good classifiers. The types of classification functions used in S1 and S2 are given in Table \ref{tab:sim_setup} along with the number of errors each of these classification functions made on the test data. From Table \ref{tab:sim_setup} we can observe that the error percentage of the best classification function is $3$ in S1, while it is $47$ in S2. A situation like in S2 can appear when the distribution of the data changes abruptly so that the classification functions trained on the old data becomes inaccurate for the new data. In our numerical results we will show how the context information can be used to improve the performance in both S1 and S2.
The accuracies of the classifiers on the test data are unknown to the learners so they cannot simply choose the best classification function. 
In all our simulations, we assume that the test data sequentially arrives to the system and the label is revealed with a one step delay. 

\begin{table}[t]
\centering
{\fontsize{8}{6}\selectfont
\setlength{\tabcolsep}{.1em}
\begin{tabular}{|l|c|c|c|c|c|}
\hline
Learner & 1 &2 &3 & 4\\
\hline
Classification  & Naive Bayes,  &Always $1$, & RBF Network, & Random Tree,  \\
Function (S1)  & Logistic & Voted Perceptron &J48 & Always $0$\\
\hline
 Error  & 47,  & 53,  & 47, & 47,  \\
percentage (S1) & 3 & 4 &  47 & 47\\
\hline
 Classification  & Naive Bayes,  &Always $1$,  & RBF Network,  & Random Tree,  \\
Function (S2) & Random & Random & J48 & Always $0$ \\
\hline
 Error & 47,  & 53,  & 47,  & 47,  \\ 
percentage (S2) & 50 & 50 & 47 & 47 \\
\hline
\end{tabular}
}
\add{\vspace{-0.1in}}
\caption{Simulation setup}
\vspace{-0.25in}
\label{tab:sim_setup}
\end{table}

Since we only consider single dimensional context, $d=1$. However, due to the bursty, non-stochastic nature of the network security data we cannot find a value $\alpha$ for which Assumption \ref{ass:lipschitz2} is true. Nevertheless, we consider two cases, C1 and C2, given in Table \ref{tab:par_setup}, for CoS parameter values.
In C2, the parameters for CoS are selected according to Theorem \ref{theorem:cos}, assuming $\alpha=1$. 
In C1, the parameter values are selected in a way that will limit exploration and training. However, the regret bound in Theorem \ref{theorem:cos} may not hold for these values. 

\begin{table}[t]
\centering
{\fontsize{8}{6}\selectfont
\setlength{\tabcolsep}{.3em}
\begin{tabular}{|l|c|c|c|c|}
\hline
 & $D_1(t)$ & $D_2(t)$ & $D_3(t)$ & $m_T$  \\
\hline
(C1) CoS & $t^{1/8} \log t$ & $2 t^{1/8} \log t$ & $t^{1/8} \log t$ & $\lceil T \rceil^{1/4}$  \\
\hline
(C2) CoS & $t^{1/2} \log t$ & $2 t^{1/2} \log t$ & $t^{1/2} \log t$ & $\lceil T \rceil^{1/4}$  \\
\hline
\end{tabular}
}
\add{\vspace{-0.05in}}
\caption{Parameters for CoS}
\label{tab:par_setup}
\add{\vspace{-0.4in}}
\end{table}

In our simulations we consider the performance of learner 1. Table \ref{tab:sim_results} shows under each simulation and parameter setup the percentage of errors made by CoS and the percentage of time steps spent in training and exploration phases for learner 1. We compare the performance of CoS with AdaBoost and sliding window AdaBoost (SWA) whose error rates are also given in Table \ref{tab:sim_results}. AdaBoost and SWA are trained using 20000 consecutive samples from the data set different from the test data. 
SWA re-trains itself in an online way using the last $w$ observations, which is called the window length. Both AdaBoost and SWA are ensemble learning methods which require learner 1 to combine the predictions of all the classification functions. Therefore, when implementing these algorithms we assume that learner 1 has access to all classification functions and their predictions, whereas when using our algorithm we assume that learner 1 only has access to its own classification functions and other learners but not their classification functions. Moreover, learner 1 is limited to use a single prediction in CoS. This may be the case in a real system when the computational capability of local learners are limited and the communication costs are high.

First, we consider the case when the parameter values are as given in C1. We observe that when the context is the previous label, CoS performs better than AdaBoost and SWA for both S1 and S2. This result shows that although CoS only uses the prediction of a single classification function, by exploiting the context information it can perform better than ensemble learning approaches which combine the predictions of all classification functions. We see that the error percentage is smallest for CoS when the context is the previous label.
This is due to the bursty nature of the attacks.
\armv{The exploration percentage for the case when context is the previous label is larger for DCZA than CoS. 
As we discussed in Section \ref{sec:reduction}, the number of explorations of DCZA can be reduced by utilizing the observations from the old hypercube to learn about the accuracy of the arms in a newly activated hypercube.
}
When the context is the feature of the data or the time, for S1, CoS  performs better than AdaBoost while SWA with window length $w=100$ can be slightly better than CoS. But again, this difference is not due to the fact that CoS makes too many errors. It is because of the fact that CoS explores and trains other classification functions and learners. AdaBoost and SWA does not require these phases. But they require communication predictions of all classification functions and communication of all local learners with each other at each time step. Moreover, SWA re-trains itself by using the predictions and labels in its time window, which makes it computationally inefficient.
Another observation is that using the feature as context is not very efficient when there are no good classifiers (S2). However, the error percentage of CoS ($39\%$) is still lower than the error percentage of the best classifier in S2 which is $47\%$.
\comment{
In Figure \ref{fig:CoS_res} and \ref{fig:DCZA_res}, the number of classification errors is given as a function of time for CoS and DCZA, respectively for S1 and C1, for all three different context information. The error function is {\em spike-shaped} because of the training phases that happens when a change in the context occurs. Note that when the context is previous label or {\em srcbytes}, the spikes follow the changes in the data, while when the context is time, the spikes are uniformly distributed. 

\begin{figure}
\begin{center}
\includegraphics[width=0.8\columnwidth]{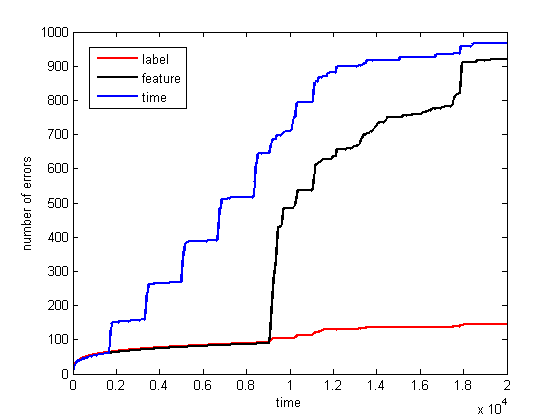}
\caption{Number of errors made by CoS as a function of time.} 
\label{fig:CoS_res}
\end{center}
\end{figure}

\begin{figure}
\begin{center}
\includegraphics[width=0.8\columnwidth]{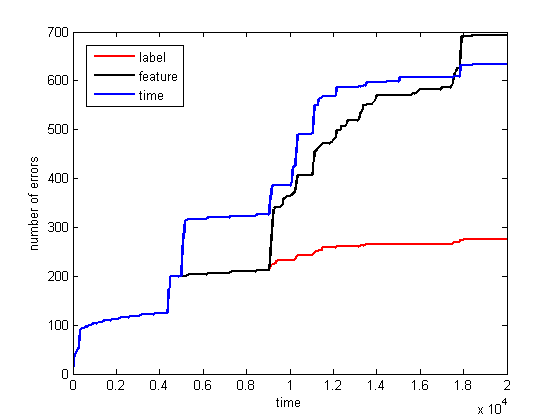}
\caption{Number of errors made by DCZA as a function of time.} 
\label{fig:DCZA_res}
\end{center}
\end{figure}
}

We observe that CoS performs poorly when the set of parameters is given by C2.
This is due to the fact that the percentage of training and exploration phases is too large for C2, thus CoS cannot exploit the information it gathered efficiently. Another important reason for the poor performance is the short time horizon. As the time horizon grows, we expect the exploration and training rates to decrease, and the exploitation rate to increase which will improve the performance.

\begin{table}[t]
\centering
{\fontsize{8}{6}\selectfont
\setlength{\tabcolsep}{.3em}
\begin{tabular}{|l|c|c|c|}
\hline
 & Error $\%$ & Training $\%$ & Exploration $\%$  \\
\hline
(C1,S1) CoS & 0.7, 4.6, 4.8 & 0.3, 3, 2.8 & 1.4, 6.3, 8.5 \\
\hline
(C1,S2) CoS & 0.9, 39, 10 & 0.3, 3, 2.8 & 1.5, 6.5, 8.6 \\
\hline
(C2,S1) CoS & 16, 14, 41 & 8.5, 16, 79 & 55 27 20\\
\hline
(S1, S2) AdaBoost & 4.8, 53 & & \\
\hline
(S1, S2) SWA ($w=100$) & 2.4, 2.7 & &\\
\hline
(S2, S2) SWA ($w=1000$) & 11, 11 & & \\
\hline
\end{tabular}
}
\add{\vspace{-0.05in}}
\caption{Simulation Results}
\label{tab:sim_results}
\vspace{-0.4in}
\end{table}

\add{\vspace{-0.15in}}
\section{Conclusion} \label{sec:conc}
In this paper we developed a novel online learning algorithm for decentralized Big Data classification using context information about the high dimensional data.
We proved sublinear regret results for this algorithm and showed via numerical results that it outperforms ensemble learning approaches in a network security application.
\rmv{An interesting research direction is to see the performance of CoS when combined with ensemble learning approaches. This will increase the communication and computation costs of the learners but the improvement in classification accuracy can be large compared to existing online ensemble learning methods such as SWA.}
\add{\vspace{-0.2in}}

\remove{
 \appendices
 \section{A bound on divergent series} \label{app:seriesbound}
 For $p>0$, $p \neq 1$,
\begin{align*}
\sum_{t=1}^{T} \frac{1}{t^p} \leq 1 + \frac{T^{1-p} -1}{1-p}
\end{align*}
\begin{proof}
See \cite{chlebus2009approximate}.
\end{proof}
}
\bibliographystyle{IEEE}
\bibliography{OSA}


\end{document}